\documentclass[oribibl]{llncs}

\usepackage[english]{babel} 
\usepackage[numbers]{natbib} 

\usepackage{amsmath} 
\usepackage{amsfonts} 

\newcommand{\w}{\omega}
\newcommand{\X}{\mathcal{X}}
\newcommand{\W}{\mathcal{W}}

\newcommand{\NP}{\textit{NP}}
\newcommand{\coNP}{\textit{coNP}}
\newcommand{\NPO}{\textit{NPO}}

\newcommand{\PSPACE}{\mbox{\textit{PSPACE}}}
\newcommand{\EXPTIME}{\mbox{\textit {EXPTIME}}}

\title{The Complexity of Campaigning}
\author{Cory Siler \and Luke Harold Miles \and Judy Goldsmith}
\institute{University of Kentucky, Lexington, Kentucky, United States \\
\email{jcsi225@g.uky.edu, luke.lambda@uky.edu, goldsmit@cs.uky.edu}
}

\begin{document}

\maketitle

\begin{abstract}
In ``The Logic of Campaigning'', Dean and Parikh consider a candidate making campaign statements to appeal to the voters. They model these statements as Boolean formulas over variables that represent stances on the issues, and study optimal candidate strategies under three proposed models of voter preferences based on the assignments that satisfy these formulas. We prove that voter utility evaluation is computationally hard under these preference models (in one case, $\#P$-hard), along with certain problems related to candidate strategic reasoning. Our results raise questions about the desirable characteristics of a voter preference model and to what extent a polynomial-time-evaluable function can capture them.
\end{abstract}

\section{Introduction}

In light of some fairly surprising election outcomes around the world, we are very interested in understanding how politicians construct their platforms. For instance, what motivates many candidates to speak in platitudes that reveal little information about their views? On the other hand, what motivates candidates to commit to specific and sometimes audacious policies? The focus of this paper is a logical formalism introduced by \citeauthor{dean2011logic} \citep{dean2011logic} (and extended by \citeauthor{parikh2015strategy} \citep{parikh2015strategy}) that aims to explain candidates' choices of campaign statements to make; these statements are modeled as propositional formulas over variables representing stances on issues, and \citeauthor{dean2011logic} consider different definitions of voters' utility for a candidate as functions of the possible sets of policies that the candidate might implement based on these statements.

Political scientists have also taken interest in candidates' decisions about what to say when campaigning; \citeauthor{petrocik1996issue} \cite{petrocik1996issue} found empirical evidence that a candidate will try to focus on issues where the candidate has a good record and their opponents have bad records.
Game theorists have a shared interest with \citeauthor{dean2011logic} in what might motivate a candidate to be ambiguous \cite{aragones2000strategic,aragones2002ambiguity,baghdasaryan2016set}.
The game-theoretic
models take into account the interaction between multiple candidates (and in the case of \citeauthor{baghdasaryan2016set}'s model \cite{baghdasaryan2016set}, voters' uncertainty about their most-preferred policies), but often with simplified representations of a platform (e.g., points on a one-dimensional spectrum or probability distributions over a small set of alternatives). In contrast, \citeauthor{dean2011logic}'s framework abstracts away details of the electoral system like multi-agent interactions and voter strategy to focus on the implications for an individual candidate of communicating using more expressive logic-based statements.

This expressivity, however, brings computational costs. In this paper, we consider the computational complexity of problems related to voter and candidate reasoning. Although \citeauthor{dean2011logic}'s formulations capture many desirable characteristics of possible voters, and go a long way toward explaining the assumptions about voters that both vague and overspecific candidates might be making, our results raise some questions about whether these are the right models of how voters evaluate candidates' platforms.

In Section \ref{sec:prelims} we introduce \citeauthor{dean2011logic}'s model and some computational complexity classes we will use later. We then find computational complexity results for problems related to the model --- in Section \ref{sec:utilityeval}, evaluating voters' utility for a candidate; in Section \ref{sec:optimizing}, choosing campaign statements that optimize total voter utility for the candidate; and in Section \ref{sec:thresh}, choosing campaign statements that motivate enough individual voters to vote for the candidate. In Section \ref{sec:conclusion}, we conclude with directions for future work in modeling campaigns, including questions about desirable characteristics of voter evaluation.

\section{Preliminaries}
\label{sec:prelims}

\subsection{Candidates, Voters, and Statements}

In \citeauthor{dean2011logic}'s model, political views are expressed in terms of Boolean variables (atomic propositions) $\X = \{x_1, \dots, x_n\}$ (e.g., $x_1=$ ``Every citizen is entitled to a free pony.'', $x_2=$ ``Tooth-brushing should be mandatory.'', $x_3=$ ``We must invest in zombie-based renewable energy sources.''). A candidate makes statements about their platform in the form of propositional formulas over these variables (e.g., $ \neg x_1$, or $x_1 \rightarrow x_2 $). 
The candidate's current \emph{theory}\footnote{Some sources, particularly in the belief revision literature, use the term \emph{belief base}.}
$T$ consists of the statements the candidate has issued so far and their logical closure. In our discussion of complexity, we will assume that $T$ is given as a set of statements and that anything in the logical closure besides the statements themselves must be computed. 
In general we assume that $T$ is self-consistent.

A voter $v$ has a preference function $p_v : \X \to [-1,1]$ indicating which direction and how strongly $v$ stands on each issue $x_i$:\footnote{Note that we modify our notation from \citeauthor{dean2011logic}'s. In particular, they represent this preference function using two quantities --- a weight in $[0,1]$ and a truth-value preference in $\{ -1, 0, 1 \}$ --- which we combine into the single function $p_v$.}
A negative $p_v(x_i)$ indicates that $v$ prefers $x_i$ to be false and a positive $p_v(x_i)$ indicates that $v$ prefers $x_i$ to be true, with the magnitude reflecting the strength of preference (and 0 being indifference). If, for example, $v$ was against mandatory tooth-brushing and cared greatly about this issue, then we might have $p_v(x_2) = -0.9$.  We assume that candidates have complete knowledge of the voters' preference functions.

We let $\W$ be the set of all possible assignments (\emph{worlds}) to the variables $\X$. Hence, $|\W| = 2^{|\X|} = 2^n$. We denote a specific world  as $\w \in \W$.
We say that $\w$ \textit{models} a theory $T$, $\w \models T$, if $\w$ is consistent with the logical closure of $T$.
For the purpose of defining voter utilities, we treat $\w$ as a function $\w: \X \to \{-1, 1\}$ where $\w(x_i) = 1$ if $x_i$ is true in that world and $\w(x_i) = -1$ if $x_i$ is false in that world.
The voter's utility for some $\w \in \mathcal{W}$ is
$$ u_v(\w) = \sum_{x_i \in \X} p_v(x_i) \cdot \w(x_i). $$

The voter's utility for a candidate is a function of the possible worlds modeled by the current theory $T$ of what the candidate has said so far. \citeauthor{dean2011logic} consider three classes of voters: 
\begin{itemize}
\item \emph{Optimistic} voters evaluate the candidate on the best world modeled by the theory, $ ut_v(T) = \textup{max} \{ u_v(\w) : \w \models T \} $. 
\item \emph{Pessimistic} voters evaluate the candidate on the worst world modeled by the theory, $ ut_v(T) = \textup{min} \{ u_v(\w) : \w \models T \} $. 
\item \emph{Expected-value} voters take the average\footnote{\citeauthor{dean2011logic} assume that for the purpose of determining ``expected value'' of the worlds, all worlds are considered equally likely.} utility over modeled worlds,\footnote{One consequence of this definition is that the utility of an empty or otherwise tautological theory is 0 for any expected-value voter.} 
$$ ut_v(T) = \frac{\sum_{\w \models T} u_v(\w)}{| \{ \w : \w \models T \} |}.$$ 
\end{itemize}

\citeauthor{dean2011logic} consider the case of a single candidate who wants to choose statements that maximize the total utility of a population of voters. They prove that with expected-value voters, an optimal strategy involves announcing a stance on every issue, creating a theory that models only one world. (We will refer to such a theory as a \emph{complete} theory; when a theory $T'$ is complete and $T \subseteq T'$, we call $T'$ a \emph{completion} of $T$.) Furthermore, they observe that with pessimistic voters, it is also advantageous for the candidate to announce a stance on every issue, since eliminating possible worlds can never result in a loss of utility. Only with optimistic voters is it advantageous to remain silent, since eliminating possible worlds can never result in a gain of utility. 

\subsection{Computational Complexity Classes}

We assume familiarity with $P$ and \NP. In addition, we will invoke some well-studied but less common complexity notions, which we describe here.

Let   $C$ and $D$ be computational complexity classes defined via resource bounds on Turing machines.  We denote by  $C^D$ those languages or functions computable by a $C$ Turing machine with an \emph{oracle} for $D$.  In other words, we modify a $C$ Turing machine to have an additional tape and state $q$.  If $d$ is a language or function computable by a $D$ Turing machine, then we allow computations of the modified Turing machine to write a string $x$ on the new tape, enter state $q$, and in the next step, the new tape contains $d(x)$. 
We write $A\leq_T^P B$ if $A\in P^B$, meaning ``$A$ is polynomial-time Turing reducible to $B$''. 

The class {\NPO} is an analogue to {\NP} for optimization problems, i.e., problems that are specified in terms of a definition of valid \emph{instances}, a definition of valid \emph{solutions} with respect to an instance, and a \emph{value function} over the solutions, and ask for a solution with maximum or minimum value.
Such a problem is in {\NPO} if and only if it  meets the following criteria defined by \citeauthor{ausiello1999complexity} \cite{ausiello1999complexity}: Instances can be verified as valid in polynomial time, solutions can be verified as valid in time polynomial in the instance size, and the value of a solution can be computed in polynomial time.  Note that $\NPO\subseteq P^{\NP}$. 

The function class $\#P$, introduced by \citeauthor{valiant1979complexity} \citep{valiant1979complexity}, contains those problems that are equivalent to determining the number of accepting paths in an {\NP} Turing machine. If a problem is in \NP, then the problem of counting how many witness strings satisfy the $\NP$ machine for a given instance is in $\#P$. Since a nonzero answer for a $\#P$ problem instance entails a positive answer for the corresponding $\NP$ problem instance and a zero answer for the $\#P$ problem entails a negative answer for the $\NP$ problem, $\NP \leq_T^P \#P$.

We have $P\subseteq \NP \subseteq P^{\NP} \subseteq P^{\#P} \subseteq \PSPACE \subseteq \EXPTIME$.  

\section{Complexity of Finding Voter Utility}
\label{sec:utilityeval}

One important factor in the epistemology of campaigns that \citeauthor{dean2011logic}'s framework (with its implicit assumption of ``logical omniscience'') does not explicitly model is cognitive complexity (for which we use \emph{computational} complexity as a proxy) as it pertains to the voters. We argue that even a superficially simple series of campaign statements may induce a complex underlying theory:  Though we, and \citeauthor{dean2011logic}, have introduced the set of variables $\X$ in terms of high-level issues ($x_1=$``Every citizen is entitled to a free pony'') for explanatory purposes, in reality such issues might more accurately be viewed as complex interplays of finer-granularity subissues (``Every citizen is entitled to a free pony'' $= ( x'_1 \vee x'_2 \vee x'_3 ) \wedge ( x'_1 \rightarrow x'_4 \vee \neg x'_5 ) \wedge \ldots  $, where $x'_1,x'_2,x'_3$ are potential taxes to fund the pony giveaway, $x'_4,x'_5$ are about the logistics of pony distribution, and so on). 

Furthermore, as \citeauthor{dean2011logic} note, additional information can arise from a statement through \emph{implicature} --- that which is suggested by a speaker without directly being part of or entailed by ``what is said''. For instance, when Vermin Supreme says, ``When I'm president everyone gets a free pony'', we discount the possibility that he plans to give everyone \emph{two} free ponies; if he did, that would not \emph{contradict} his promise, but his omission of information would be \emph{infelicitous}.\footnote{In \citeauthor{Grice1975-GRILA-2}'s account of implicature \citep{Grice1975-GRILA-2}, participants in a conversation assume each other to be obeying certain maxims of cooperativity (for instance, illustrated here is the maxim of Quantity --- roughly, \emph{give as much information as necessary, and do not give more information than necessary}); they interpret each other's statements in light of this mutual assumption.}  

A potential source of discrepancies between the framework's predictions and the reality of campaigns is that, given an elaborate body of information about a candidate's policies, voters have trouble evaluating the candidate due to the intractability of their utility functions.
We will consider the computational complexity of the function problems of determining exact voter utility, but also of decision problems of determining whether the utility meets a given threshold, 
which are particularly relevant for the ``stay-at-home voter'' scenario we will discuss in Section \ref{sec:thresh}.

\subsection{Optimistic Voter Evaluation}

\begin{theorem}
\label{thm:opt-thresh}
Given a theory $T$, an optimistic voter $v$, and a value $k$, the problem of deciding whether $ut_v(T) \geq k$ is \NP-complete.
\end{theorem}

\begin{proof}
For $\NP$ membership, observe that given a world modeled by $T$ for which $v$'s utility is at least $k$, we can verify the consistency and utility in polynomial time.

We will show \NP-hardness with a polynomial-time reduction from Boolean satisfiability (SAT). Let $\phi$ be a propositional formula over a set of variables $\X$.  We construct the theory as $T = \{ x_* \rightarrow \phi \}$, where $x_*$ is a new variable.  We construct an optimistic voter $v$ with preferences set as $p_v(x_*) = 1$ and $p_v(x_i) = 0$ for all $x_i \in \X$. And we let $k = 1$. Let $A = \{\w : \w \models T\}$. If $\phi$ is unsatisfiable, then $A = \{\w : \w(x_*)=-1\}$, hence $ut_v(T) = -1$. However, if $\phi$ is satisfiable then there are some $\w \in A$ where $w(x_*)=1$ and hence $ut_v(T) = 1$. Finally, we have $ut_v(T) \geq 1 = k$ if and only if $\phi$ is satisfiable.
\end{proof}

\begin{theorem}
\label{thm:opt-exact}
Given a theory $T$ and an optimistic voter $v$, the problem of computing $ut_v(T)$ and a corresponding best world modeled by $T$ is \NPO-complete. 
\end{theorem}

\begin{proof}
The problem satisfies the criteria for \NPO-membership \citep{ausiello1999complexity}: Instances (i.e., the theory and voter specification) and solutions (i.e., worlds modeled by the theory) are recognizable as such in time polynomial in the instance size, and the value function (i.e., voter utility) is computable in polynomial time.

We will show \NPO-hardness with a polynomial-time reduction from the maximum weighted satisfiability problem (MAX-WSAT),\footnote{The name ``weighted satisfiability'' (WSAT) has been used by different sources to refer to two different groups of problems --- one where an instance consists only of a propositional formula and the value of a solution is the number of true variables (the Hamming weight), and the generalization we use here where the instance includes weights for the variables. The maximization/minimization versions of the former are sometimes called ``maximum number of ones'' (MAX-ONES) / ``minimum number of ones'' (MIN-ONES), and are complete for \emph{\NPO-PB} \citep{panconesi1993quantifiers}, a subclass of {\NPO} where the magnitude of a solution's value is polynomially bounded by the size of the input.}
for which \citeauthor{ausiello1999complexity}  \citep{ausiello1999complexity} prove \NPO-completeness. A MAX-WSAT instance consists of a propositional formula $\phi$ and a positive weight $r_i$ for each variable $x_i$; the problem is to find a satisfying assignment that maximizes the total weight of the variables assigned to be true. Let $R = \max \{r_i : 1 \leq i \leq n\}$. We construct the theory as $T=\{\phi\}$ and the voter preferences as
$p_v (x_i) = r_i / R$ for each $x_i$.\footnote{We divide through by the maximum weight so that the $p_v(x_i)$'s are in $[0,1]$.
} Then the voter's best world $\w$ is the optimal assignment for the MAX-WSAT instance, and given $u_v(\w) \in \left[ -\sum_i r_i/R, \sum_i r_i/R \right]$ we can retrieve the corresponding total weight for the MAX-WSAT assignment by mapping this range onto $\left[ 0, \sum_i r_i \right]$. 
\end{proof}

\subsection{Pessimistic Voter Evaluation}

\begin{theorem}\label{thm:pes-thresh}
Given a theory $T$, a pessimistic voter $v$, and a value $k$, the problem of deciding whether $ut_v(T) \geq k$ is \coNP-complete.
\end{theorem}

\begin{proof}
For {\coNP} membership, observe that given a world modeled by $T$ for which $v$'s utility is less than $k$, we can verify the consistency and utility in polynomial time.

We can show \coNP-hardness by polynomial-time reduction to this problem from Boolean unsatisfiability (UNSAT). Given a formula $\phi$, we construct the theory $T$ and the voter $v$'s preferences in the same way as in the proof of Theorem \ref{thm:opt-thresh}, except that $v$ prefers the new variable $x_*$ to be false, $p_v(x_*)=-1$. Then $ut_v(T) \geq 1 = k$ if and only if $\phi$ is unsatisfiable.

\end{proof}
\begin{theorem}
\label{thm:pes-exact}
Given a theory $T$ and a pessimistic voter $v$, the problem of computing $ut_v(T)$ and a corresponding worst world modeled by $T$ is \NPO-complete. 
\end{theorem}

\begin{proof}
\NPO\ membership applies by the same argument as in the proof of Theorem \ref{thm:opt-exact}. We can show \NPO-hardness by polynomial-time reduction from the minimum weighted satisfiability problem (MIN-WSAT), the minimization counterpart to MAX-WSAT; the mapping is constructed in the same manner as in the proof of Theorem \ref{thm:opt-exact}.
\end{proof}

\subsection{Expected-Value Voter Evaluation}

\begin{lemma} \label{sharp-hard}
Given a theory $T$ and an expected-value voter $v$, the problem of computing $ut_v(T)$ is $\leq_T^P$-hard for $\#P$.
\end{lemma}

\begin{proof}
Let $\phi$ be a Boolean formula over $\{x_1,\ldots,x_n\}$ and let $S = \#SAT(\phi)$ be the number of satisfying assignments of $\phi$. We define a formula $\psi'$ over $\{x_1,\ldots,x_n\}\cup \{y,z\}$ as follows. First we define $\psi = (\phi \land y \land z)$, and set
$$
\psi' = \psi \lor (y \land \neg z \land \bigwedge^n_{i=1}x_i).
$$
Define a voter $D$ with preferences $p_D(y)=p_D(z) = 1$  and $p_D(x_i)=0$ for each $x_i \in \{x_1, \dots, x_n\}$. (So this voter is defined over $n+2$ variables.) 

Let $A =\{-1,1\}^n $ and $B = \{-1,1\}^{n+2} $ be the possible worlds for $\phi$ and $\psi$, respectively. Then we have
$$
S
  =| \{ a \in A : a \models \phi \} |
  =| \{ b \in B : b \models \psi \} |
  =| \{ b \in B : b \models \psi' \} | - 1.
$$

That is, $\phi$ has as many satisfying assignments over $n$ variables as $\psi$ has over $n+2$ variables, and $\psi'$ has one more satisfying assignment than $\psi$. Then we get the critical equalities:
\begin{align*}
  \frac {ut_D(\{\psi\})} {ut_D(\{\psi'\})}
  &= \frac {(\sum_{b \models \psi} u_D(b))/|\{b \in B : b \models \psi\}|} {(\sum_{b \models \psi'} u_D(b))/|\{b \in B : b \models \psi'\}|} \\
  &= \frac {(\sum_{b \models \psi} u_D(b))/S} {(\sum_{b \models \psi'} u_D(b))/(S+1)} \\
  &= \frac {(\sum_{b \models \psi} u_D(b))/S} {(0+\sum_{b \models \psi} u_D(b))/(S+1)}
  = \frac{S+1}{S}.
\end{align*}

(The third equality is valid because the only world in $\{b \in B : b \models \psi'\} \setminus \{b \in B : b \models \psi\}$ has utility 0 for voter $D$.)
This allows us to derive an equation to get \#SAT from $ut_D$:
$$
\frac {1} {ut_D(\psi)/ut_D(\psi') - 1}
= \frac {1} {(S+1)/S - S/S}
= \frac {1} {1/S}
= S
= \#SAT(\phi).
$$

Thus, we can use two calls to an oracle for expected-value utility to compute $\#SAT(\phi)$ in polynomial time.
\end{proof}

\begin{lemma} 
\label{in-sharp}
Given a theory $T$ and an expected-value voter $v$, the problem of computing $ut_v(T)$ is in $P^{\#P}$.
\end{lemma}

\begin{proof}
Let $T$ be some theory and let $v$ be some expected-value voter. We assume $p_v$ is represented as a vector of rational binary numbers and we set $b$ as the number of bits in the `longest' number. We describe an {\NP} Turing machine $M$ whose number of witness strings is proportional to $u_v(T)$.

The machine $M$ takes in a Boolean formula $\phi$ and a voter's preference function $p_v$. Then $M$ guesses an assignment $\w$ and a binary integer $k$ where $ 1 \leq k \leq 2^b$. If both $\w \models \phi$ and $k \leq u_v(\w) \cdot 2^b$, then the machine accepts. Otherwise, the machine rejects.
Note that both checks take polynomial time. 

Given $\phi$ and $b$ and $p_v$, how many ways can $M$ accept? If $\w$ does not satisfy, then $M$ cannot accept. If $\w$ does satisfy, then $M$ accepts in exactly $u_v(\w) \cdot 2^b$ different ways. Hence, $\#M(\phi, p_v, b) = 2^b \sum_{w \models \phi} u_v(\w)$. Finally, we can compute the utility:
$$
ut_v(T) = \frac{\#M(\bigwedge T, p_v, b)}{\#SAT(\bigwedge T) \cdot 2^b}.
$$
\end{proof}

\begin{theorem}
Given a theory $T$ and an expected-value voter $v$, the problem of computing $ut_v(T)$ is $\leq_T^P$-complete for $\#P$. 
\end{theorem}
\begin{proof}
This follows from Lemmas \ref{sharp-hard} and \ref{in-sharp}.
\end{proof}

\section{Complexity of Making an Optimal Theory}
\label{sec:optimizing}

\citeauthor{dean2011logic} observe that when all voters are optimistic, a candidate looking to increase total voter utility is best off simply saying nothing; as such, this situation does not raise any nontrivial computational issues from the candidate's perspective. On the other hand, when appealing to an expected-value or pessimistic voter population, the candidate is best off taking an explicit stance on every issue. The candidate's ability to do so, of course, depends on their knowing \emph{which} stance to take. This has two aspects: Firstly, the candidate must know the voters' stances on each issue; given the increasing availability of tools like mass surveys and data analytics that let politicians gauge the attitudes of their constituents, this is a reasonable assumption (though models of candidate uncertainty about voter stances are of interest for future study). Secondly, the candidate must be able to compute the best announcements for appealing to the overall voter population, given the individual voter preferences; this is the family of problems we examine here. When we say ``optimal theory'' in the following results, we mean a theory $T$ that maximizes $\sum_{v \in V} ut_v(T)$ for the voter population $V$.

In general, we assume that the candidate starts with an empty theory, and that the candidate is willing to craft whatever platform is most advantageous (being what \citeauthor{dean2011logic} call a ``Machiavellian'' candidate) rather than being committed to personal beliefs. However, in Section \ref{sec:extending}, we show that having to remain consistent with an existing theory raises the complexity of some relevant problems.

\subsection{Appealing to Expected-Value Voters}

\begin{theorem}
\label{thm:mach-optimal-from-scratch}
Given $n$ variables and a set $ V = \left\lbrace v_1, \cdots, v_m \right\rbrace $ of expected-value voters, a candidate can construct an optimal theory in time $O(n \cdot m)$.
\end{theorem}
\begin{proof}
In particular, we will describe a procedure for finding a \emph{complete} optimal theory (since \citeauthor{dean2011logic} have established that with expected-value voters there always exists a complete theory that is optimal).

We can reformulate the set $V$ of $m$ voters with possibly many different preference functions into a new set $V'$ of $m$ voters with all the same preference function such that the candidate receives the same total utility, i.e., $\sum_{v \in V} ut_v(T) = \sum_{v' \in V'} ut_{v'}(T)$ for any theory $T$. 
We define
\[
p_{v'}(x_i)= \sum_{v \in V} \frac{p_{v}(x_i)}{|V|}
\]
for each variable $x_i$, for each $v' \in V'$.
This reformulation takes time $O(m)$ to compute for each of the $n$ variables. The candidate can then construct a theory that models only the world with the preferred assignment to each variable according to the new preference function.
\end{proof} 

\subsection{Appealing to Pessimistic Voters}

\begin{theorem}
Given $n$ variables and a set $ V = \left\lbrace v_1, \cdots, v_m \right\rbrace $ of pessimistic voters, a candidate can construct an optimal theory in time $O(n \cdot m)$.
\end{theorem}
\begin{proof}
The procedure from the proof of Theorem \ref{thm:mach-optimal-from-scratch} can also be used to construct a complete optimal theory for pessimistic voters; while the equality between the total utilities of the original and reformulated sets of voters no longer holds in general, it still holds for theories like the constructed one that model only a single world (since the utility for this world is both the pessimistic and expected value for the theory).
\end{proof}

\subsection{Extending an Existing Theory}
\label{sec:extending}
Until now, we have assumed that a candidate starts with a ``blank slate'', able and willing to shape the voters' beliefs with no restrictions. However, there are many reasons why the candidate may instead need to stay consistent with particular set of formulas --- the candidate may be an experienced politician who has revealed platform information in prior elections and incumbencies, may be a member of a political party with established doctrine, or may be ``tactically honest'' \cite{parikh2015strategy} --- willing to make strategic statements only insofar as they do not contradict certain deeply-held opinions. The strategy of choosing the most informative theory possible to appeal to expected-value or pessimistic voters becomes harder when the candidate must also remain consistent with an existing theory:

\begin{theorem}
Given an expected-value or pessimistic voter $v$ and current theory $T$, the problem of computing an optimal completion of $T$ is \NPO-complete. 
\end{theorem}
\begin{proof}
Observe that this problem is equivalent to the optimistic voter utility problem from Theorem \ref{thm:opt-exact} (except that instead of yielding a best world $\w$ modeled by a theory, we are yielding a \emph{theory} that models only $\w$, which can be accomplished by inserting into $T$ a conjunction of literals with their assignments in $\w$); thus, the proof of Theorem \ref{thm:opt-exact} applies here as well.
\end{proof}

\section{Complexity of Motivating Enough Voters to Vote}
\label{sec:thresh}

While having an enthusiastic constituency is no doubt correlated with a candidate's success, the more direct measure is whether enough supporters actually turn up to vote.
A 2006 Pew Research Center study addressed the question of when people vote --- since so many people do not, at least in the US.  Their findings included the following, which addresses the questions of showing up, rather than the decision about how to vote:
\begin{quotation}
The Pew analysis identifies basic attitudes and lifestyles that keep these intermittent voters less engaged in politics and the political process. Political knowledge is key: Six-in-ten intermittent voters say they sometimes don't know enough about candidates to vote compared with 44\% of regular voters – the single most important attitudinal difference between intermittent and regular voters identified in the survey.
[...]
One other key difference: Regular voters are more likely than intermittent voters to say they have been contacted by a candidate or political group encouraging them to vote, underscoring the value of get-out-to-vote campaigns and other forms of party outreach for encouraging political participation.
\end{quotation}
\begin{flushright}
\citeauthor{PewWhyVote} \cite{PewWhyVote}
\end{flushright}

Refusal to vote does not necessarily indicate irrationality on the voter's part; under decision-theoretic models of \emph{expressive voting} \cite{aragones2011making,klein2013expressive}, where a voter's foremost goal is expressing their views rather than bringing about an outcome,  abstinence from voting is a rational choice under certain circumstances.
\citeauthor{parikh2015strategy} \citep{parikh2015strategy} suggest the presence of ``stay-at-home voters'', whose utility for a candidate must meet a certain threshold before they will vote, to explain why a candidate might remain silent in situations where \citeauthor{dean2011logic}'s model would otherwise suggest a strategy of explicitness.

\subsection{Appealing to Optimistic Voters}

\begin{theorem}
Given an integer $h$ and a set of optimistic voters $V= \{ v_1, \cdots, v_m \}$ with thresholds $\{ k_1, \cdots, k_m \}$, the problem of deciding the existence of a theory $T$ such that $ut_{v_i}(T) \geq k_i$ for at least $h$ voters is in $P$.
\end{theorem}
\begin{proof}
Since the empty theory $T=\emptyset$ has the maximum utility for any optimistic voter, it suffices to compute the utility of each voter $v_i$'s best world, $\sum_{x_j \in \X} |p_{v_i}(x_j)|$, and check whether at least $h$ of these utilities meet their respective voters' thresholds.
\end{proof}

\subsection{Appealing to Pessimistic Voters}

\begin{theorem}
\label{thm:mach-decide-threshold}
Given an integer $h$ and a set of pessimistic voters $V= \{ v_1, \cdots, v_m \}$ with thresholds $\{ k_1, \cdots, k_m \}$, the problem deciding the existence of a theory $T$ such that $ut_{v_i}(T) \geq k_i$ for at least $h$ voters is {\NP}-complete.
\end{theorem}
\begin{proof}
For {\NP} membership: If there exists a theory for which at least $h$ pessimistic voters meet their thresholds, then these voters also meet their thresholds in any completion of this theory (since eliminating worlds never decreases pessimistic voter utility); given the world modeled by one of these completions, we can verify in polynomial time that the thresholds are met.

We will show {\NP}-hardness with a reduction from conjunctive normal form Boolean satisfiability (CNF-SAT). Let $\phi$ be a Boolean formula in conjunctive normal form. For each clause containing $r$ literals, we construct a pessimistic voter $v_i$ with preferences as follows: For each variable $x_j$, $p_{v_i}(x_j) = \frac{1}{r}$ if $x_j$ appears in the clause with positive polarity, $p_{v_i}(x_j) = -\frac{1}{r}$ if $x_j$ appears in the clause with negative polarity ($\neg x_j$), and $p_{v_i}(x_j) = 0$ if $x_j$ does not appear in the clause. We set $v_i$'s threshold as $k_i = -\frac{r-1}{r}$ so that $ut_{v_i}(T) \geq k_i$ if and only if all worlds modeled by $T$ have at least one variable assigned to match its polarity in the clause. Finally, we set $h$ equal to the number of clauses to require that all worlds modeled by $T$ have at least one variable assigned to match its polarity for \emph{every} clause; such a $T$ exists if and only if $\phi$ is satisfiable.
\end{proof}

\subsection{Appealing to Mixed Voters}

\begin{theorem}
\label{thm:mach-decide--threshold-optandpess}
Given an integer $h$, a set of voters $V = \{ v_1, \cdots, v_m \} = V^o \cup V^p$ where $V^o$ consists of optimistic voters and $V^p$ consists of pessimistic voters, and thresholds $\{ k_1, \cdots, k_m \}$, the problem of deciding the existence of a theory $T$ such that $ut_{v_i}(T) \geq k_i$ for at least $h$ voters is {\NP}-complete.
\end{theorem}
\begin{proof}
{\NP}-hardness follows from the fact that this is a generalization of the problem with only pessimistic voters from Theorem \ref{thm:mach-decide-threshold}.

For {\NP}-membership: 
Let $V' \subseteq V$ be a set of $h$ or more voters. We could guess an acceptable world (not necessarily distinct) $\w_i$ for each optimistic voter $v_i \in V' \cap V^o$ (i.e., $u_{v_i}(\w_i)>k_i$), such that $\w_i$ is also acceptable for each pessimistic voter $v_j \in V' \cap V^p$ (i.e., $u_{v_j}(\w_i)>k_j$). Then the disjunction $T = \{ \vee_i w_i \}$ would satisfy all the voters in $V'$, and each voter can verify this in polynomial time, since there will only be as many modeled worlds as there are optimistic voters. 
Furthermore, if there exist  theories that satisfy $h$ or more voters, then there exists at least one in the aforementioned form.
\end{proof}

\section{Conclusions}
\label{sec:conclusion}

There are many bodies of research that are, or might be, relevant to whether voters show up to vote, and once there, how they vote.  We have explored the computational side of one such theory, and showed that it proposes computationally intensive methods for voter evaluation of platforms and for multi-voter satisficing.  Given the basic premise that voter satisfaction or satisficing is a combinatorial problem, the intractability is not surprising.

There are many ways this investigation of the computational complexity of modeling voters' behavior can be extended.  They include:
\begin{itemize}
\item adding value to informativeness of the candidate(s)' platform when predicting whether a voter will show up to vote, as per the Pew study \cite{PewWhyVote};
\item decreasing the complexity of candidates' platforms (conjunctions or disjunctions of atomic propositions; Horn formulas; \ldots);
\item modeling change over time in  voter priorities \cite{klein2014focusing} or opinions;
\item adding affective variables to the voter models \cite{parker2010vote};
\item investigating social-network models of voter interaction and influence \cite{bond201261}; 
\item using game-theoretic models of candidate-candidate interactions and voter choices \cite{krueger2008game};

\item including group-based identity in the decision to show up as well as the choice of candidate.
\end{itemize}

In addition, we could start from axiomatic characterizations of models of candidate platforms and voter choice.  What are good properties to model?  (For instance, if a candidate adds to the specificity of their platform in ways that agree with a voter's preferences, should that increase the likelihood that a voter chooses that candidate, or the likelihood that the voter shows up to vote?) Will we run into Arrow-style impossibility results for achieving all the desiderata we propose? 

\subsubsection{Acknowledgments}
The authors thank three anonymous reviewers for their helpful feedback and thank Alec Gilbert for catching errors in a late draft. All remaining errors are the responsibility of the authors.
This material is based upon work partially supported by the National Science Foundation under Grants No.~IIS-1646887 and No.~IIS-1649152.   Any opinions, findings, and conclusions or recommendations expressed in this material are those of the authors and do not necessarily reflect the views of the National Science Foundation.

\begin{quotation}
I will promise your electorate heart anything you desire, because you are my constituents, you are the informed voting public, and I have no intention of keeping any promise that I make. 
\end{quotation}
\begin{flushright}
Vermin Supreme\footnote{From Revolution PAC's 2012 interview (\url{https://www.youtube.com/watch?v=9jKszduiK8E})}
\end{flushright}

\bibliographystyle{plainnat}
\bibliography{campaigning}

\end{document}